\documentclass[10pt,twocolumn,letterpaper]{article}

\usepackage{iccv}
\usepackage{times}
\usepackage{epsfig}
\usepackage{graphicx}
\usepackage{amsmath}
\usepackage{amssymb}
\usepackage{multirow}
\usepackage{float}
\usepackage{authblk}

\newtheorem{theorem}{Theorem}

\newtheorem{definition}{Definition}
\newtheorem{proof}{Proof}

\usepackage[breaklinks=true,bookmarks=false]{hyperref}

\iccvfinalcopy 

\def\iccvPaperID{****} 

\ificcvfinal\fi

\begin{document}

\title{Exploring the Benefits of Visual Prompting in Differential Privacy}

\author[1]{Yizhe Li}
\author[2]{Yu-Lin Tsai}
\author[2]{Chia-Mu Yu}
\author[3]{Pin-Yu Chen}
\author[1]{Xuebin Ren}
\affil[1]{School of Computer Science and Technology, Xi'an Jiaotong University}
\affil[2]{National Yang Ming Chiao Tung University}
\affil[3]{IBM Research}

\maketitle
\ificcvfinal\thispagestyle{empty}\fi


\begin{abstract}
Visual Prompting (VP) is an emerging and powerful technique that allows sample-efficient adaptation to downstream tasks by engineering a well-trained frozen source model. In this work, we explore the benefits of VP in constructing compelling neural network classifiers with differential privacy (DP). We explore and integrate VP into canonical DP training methods and demonstrate its simplicity and efficiency. In particular, we discover that VP in tandem with PATE, a state-of-the-art DP training method that leverages the knowledge transfer from an ensemble of teachers, achieves the state-of-the-art privacy-utility trade-off with minimum expenditure of privacy budget. Moreover, we conduct additional experiments on cross-domain image classification with a sufficient domain gap to further unveil the advantage of VP in DP. Lastly, we also conduct extensive ablation studies to validate the effectiveness and contribution of VP under DP consideration. Our code is available at \url{https://github.com/EzzzLi/Prompt-PATE}.
\end{abstract}

\section{Introduction}\label{sec:intro}
Originating from the domain of deep learning for natural language processing, prompt engineering has gained significant popularity as an emergent technique for efficient adoption and adaptation of pre-trained language models for solving different downstream tasks \cite{liu2023pre}. In recent years, the notion of prompting has been extended to other domains and data modalities, especially in computer vision and images \cite{jia2022visual,bahng2022visual}. 
Specifically, the term \textit{visual prompting} (VP) has been coined by \cite{bahng2022visual}, and the authors show competitive accuracy of VP on some downstream image classification tasks over linear probing (i.e., attaching a trainable linear head to a pre-trained model) when used with a large vision model such as CLIP \cite{radford2021learning} (only the image encoder). It is worth noting that VP in \cite{bahng2022visual} can be viewed as a special case of \textit{model reprogramming} (MR) \cite{chen2022model} on a pre-trained model. MR inserts an input transformation layer and an output mapping layer into a pre-trained frozen model for fine-tuning downstream tasks. MR is equivalent to VP in \cite{bahng2022visual} when the input transformation is a trainable input perturbation and the output mapping is a specified source-target label correspondence or a set of text prompts for label inference (e.g., ``a photo of [predicted class]''). Throughout this paper, for ease of elucidation, we will use VP and MR interchangeably.

VP has been extensively studied for various use cases, ranging from image classification \cite{bahng2022visual}, enhancing adversarial robustness \cite{chen2022visual}, image-inpainting \cite{bar2022visual}, cross-domain adaptation \cite{tsai2020transfer,neekhara2022cross}, to name a few. In this paper, we explore yet another benefit of VP with pre-trained models -- deep learning with differential privacy (DP). In deep learning, scaling the training parameters of a neural network often leads to improved task performance (e.g., a classification model with higher accuracy) \cite{kaplan2020scaling}. However, with a DP budget, training a larger neural network usually means more consumption of data privacy \cite{luo2021scalable}. Motivated by this dilemma of the tradeoff between neural network capacity and DP, we aim to study the following fundamental question:
\begin{center}
    \textit{Can VP with a pre-trained model (trained on non-private data) improve the privacy-accuracy tradeoff in off-the-shelf DP-training mechanisms?}
\end{center}
In this paper, we give an affirmative answer to this question, validated through a comprehensive analysis and empirical comparisons. We purposely focus on existing DP-training mechanisms, in order to study the benefit of improved performance contributed by VP. Our proposed approach applies VP (at data inputs) to off-the-shelf DP-training mechanisms, together with a pre-trained model trained on non-private data. 
Particularly, when VP is used in PATE (Private Aggregation of Teacher Ensembles) \cite{pate}, a DP training mechanism, we show that the classification accuracy under a privacy constraint can achieve the current state-of-the-art performance (SOTA) (over 97\%) on the common benchmark of CIFAR-10 classification task. Furthermore, we also demonstrate that the performance increases with minimum expenditure of privacy budget. Consequently, our results uncovered new benefits of VP in DP and offer new use cases and insights into prompt engineering.

\noindent \textbf{Contribution.} We highlight our main contributions as follows. We are the first to explore the benefits of VP with pre-trained models in the design of DP classifiers. By leveraging VP, we present Prom-PATE as a training strategy for DP classifiers. While sophisticated backbones are usually difficult to be used in DP training, Prom-PATE has great flexibility in utilizing the high accuracy of the backbone without compromising privacy. Overall, Prom-PATE enjoys the following characteristics. Prom-PATE relies on VP to resolve the demand for huge data from PATE, improving practicality and accuracy. In the design, the public pre-trained models are utilized \textit{twice}, significantly growing the accuracy. Through extensive experiments, we demonstrate that Prom-PATE outperforms current DP classifiers on CIFAR-10, showing an accuracy 97.07\% under a privacy budget of $\epsilon=1.019$. We also show significant accuracy gain  of Prom-PATE in other datasets over existing methods.

\section{Related Work and Background}\label{sec:related}
\paragraph{Visual Prompting (VP) and Model Reprogramming (MR).} Both VP and MR focus on the problem setup of reusing a pre-trained model to perform a new task (either in-domain or cross-domain) without changing the model weights during fine-tuning (i.g., the pre-trained model is ``frozen''). 
MR was first studied through the lens of adversarial machine learning (ML). Elsayed et al. \cite{elsayed2018adversarial} showed that an attacker can ``steal'' an ML model's computation resource to perform another task without the model owner's consent. Later on, MR was shown to deliver competitive image classification results in data-limited and cross-domain settings \cite{tsai2020transfer,neekhara2022cross}, wherein the authors demonstrated the possibility of reusing a pre-trained model from a source domain (e.g., general image classifiers or language models) to solve challenging image classification problems in a target domain (e.g., bio-medical measurements). We refer the readers to the survey paper of MR in \cite{chen2022model} for more details. VP through a trainable (padded) universal input perturbation is revisited in \cite{bahng2022visual}, and the authors showed competitive results on some subset of 12 image classification tasks over linear probing and full fine-tuning on pre-trained image classifiers and the CLIP model \cite{radford2021learning}. Chen et al. \cite{chen2023understanding} improved VP by introducing iterative label mapping during training. Beyond image classification, VP was extended to image inpainting tasks \cite{bar2022visual}. In this paper, we note that we limit the scope of VP to input-level prompt engineering as studied in \cite{bahng2022visual,chen2022visual}, and we leave the broader notion of VP via injecting trainable token embeddings (e.g., the visual prompt tuning as in \cite{jia2022visual}) to different layers of a pre-trained model as future work.

\paragraph{Differentially Private Classifiers.} 
One of the most widely used techniques to achieve DP deep learning is DPSGD~\cite{abadi2016deep}, where DP noise is added to the clipped gradient updates during the training process. The definition and properties of DP are provided in the Supplementary Material. DPSGD suffers from information loss due to the fact that the gradient clipping and the noise scale are proportional to the norm of clipped gradient. Recent research~\cite{Chourasia2021DifferentialPD, Nasr2021AdversaryIL} finds that we may overestimate the privacy loss for DPSGD because the attacker does not have access to the gradient in each training iteration. One of the current trends in training a DP classifier is to privately fine-tune large pre-trained models such as BERT variants and GPT-2 \cite{Large-scale-private-learning, yu2022differentially, Large-Language-Models}. This private fine-tuning strategy can also be applied to the realm of images~\cite{mixedDP, luo2021scalable, Klause2022, tramer2020differentially, de2022unlocking, bu2022scalable}. For example, Tram{\`e}r and Boneh~\cite{tramer2020differentially} improved the model utility by conducting private fine-tuning with SimCLR features~\cite{simclr}. De et al.~\cite{de2022unlocking} also pre-trained the model with the public data. After that, they apply many techniques including large batch size and weight standardization to improve accuracy. Bu et al.~\cite{bu2022scalable}'s DP classifier relies on the notion of ghost clipping to calculate the clipped gradient required by DPSGD.  

PATE~\cite{pate, scalablepate} is another approach that trains a DP classifier. In PATE, the sensitive dataset is first partitioned into slices, with each \textit{teacher model} trained on a different slice of the data (through SGD). Then, the non-sensitive samples labeled by the DP noisy votes from teacher models are used to train a \textit{student model}, which turns out to be a DP classifier. Compared to DPSGD, fewer research efforts are put into the improvement of PATE. For example, Private-kNN~\cite{private-knn} relies on the private release of k-nearest neighbor (kNN) queries to avoid splitting the training set in PATE.  


\paragraph{Visual Prompting with DP.}
A recent work that combines VP and DP is Reprogrammable-FL~\cite{arifreprogrammable}.  Reprogrammable-FL is designed for DP federated learning (FL). More specifically, Reprogrammable-FL considers multiple clients, each with a common pre-trained model in each server-client interaction. The aim is to learn privatized visual prompts and label mappings for each client using DPSGD~\cite{abadi2016deep}, enabling DPFL with more efficient use of the privacy budget. Reprogrammable-FL outperforms methods that rely on private fine-tuning from pre-trained models, currently considered the standard for achieving high accuracy in DPFL. However, in each training round of Reprogrammable-FL, the update of visual prompts and label mapping for each client is still subject to clipped noisy gradient updates to ensure privacy. As a result, the overall performance may still degrade compared to the non-private setting of visual prompting~\cite{bahng2022visual}, as will be demonstrated in this paper.

\section{Main Approach}\label{sec:main}
In this section, we aim to investigate how VP can improve the privacy-utility trade-off of deep learning models. 

\paragraph{Notations.} As VP was originally proposed for model re-utilization, we denote a source model $f_{S}(\theta_{S}; x)$ which is trained from a large, source (pubic) dataset $D_{S}:= \{(x_{S}, y_{S})\}$ with $x_{S}$, where $x_{S}$ denotes the feature and $y_{S}$ denotes the label, both from the source domain. We denote our target (private) dataset $D_{T}:= \{(x_{T}, y_{T})\}$ with $x_{T}$ in which we re-utilize model $f_{S}(\theta_{S}; x)$ to accomplish the task in $D_{T}$ via VP without modifying the weights $\theta_{S}$.

\subsection{Design Challenges for DP Classifiers}\label{subsec:pate}
Though PATE outperforms DPSGD because of the reduced noise scale and no information loss from the gradient clipping, we identify three challenges for designing DP classifiers based on PATE. 
\begin{itemize}
\item \textbf{(C1)} The performance of PATE is sensitive to data partitioning. In particular, the teacher models may perform badly when the sensitive data is limited in size. As also shown in \cite{private-knn}, each teacher model has an accuracy under 50\% due to only 200 images for each partition, given 250 teacher models for CIFAR-10. One might leverage transfer learning (TL), as suggested in \cite{luo2021scalable}, to train teacher models in PATE. Specifically, this involves using a public pre-trained model and fine-tuning it on the private dataset. However, Table~\ref{tab:baseline_concept} shows that this TL-based method leads to inefficient performance in PATE\footnote{The poor accuracy of the TL-based method can be attributed to the over-partitioning of the sensitive data. In such a case, data are insufficient for the training of each teacher model.}.

\item \textbf{(C2)} A current trend in training a high-accuracy classifier in a DP manner is to take advantage of either public labeled data or a public pre-trained model. For example, De et al.~\cite{de2022unlocking} pre-train the model with ImageNet (seen as a public dataset) and then fine-tune the model with CIFAR-10 (seen as a private dataset) through DPSGD. De et al. achieve the predicting accuracy 94.7\% under $\epsilon=1$. While many pieces of evidence show that properly exploiting public datasets and models may significantly improve accuracy, a natural question that arises is whether exploiting public datasets and models more times in the design of DP classifiers benefits accuracy. 

\item \textbf{(C3)} Privately training a model pre-trained on the public dataset is a promising solution for DP-classifiers. However, take ImageNet and CIFAR-10 as examples. They may share a similar distribution and so make the above training strategy doubtful in DP guarantee \cite{tramer2022considerations}.
\end{itemize}

\begin{table}[h!]
\vspace{-3mm}
    \begin{center}
    \resizebox{\columnwidth}{!}{
    \begin{tabular}{|c || c c|} 
     \hline
     \textbf{CIFAR-10} & \textbf{Prom-PATE (ours)} & \textbf{TL-based method}  \\ [0.5ex] 
     \hline
     $\epsilon$ & 1.019 & 1.021 \\ 
     \hline 
     \textbf{Accuracy ± Std(\%)} & \textbf{97.07 $\pm$ 0.50} & 76.93 $\pm$ 0.81 \\
     \hline
    \end{tabular}
    }
    \end{center}
    \caption{Comparison of Prom-PATE and TL-based method.}
    \label{tab:baseline_concept}
\end{table}

 


\begin{figure*}
    \centering
    \includegraphics[width=0.8\textwidth]{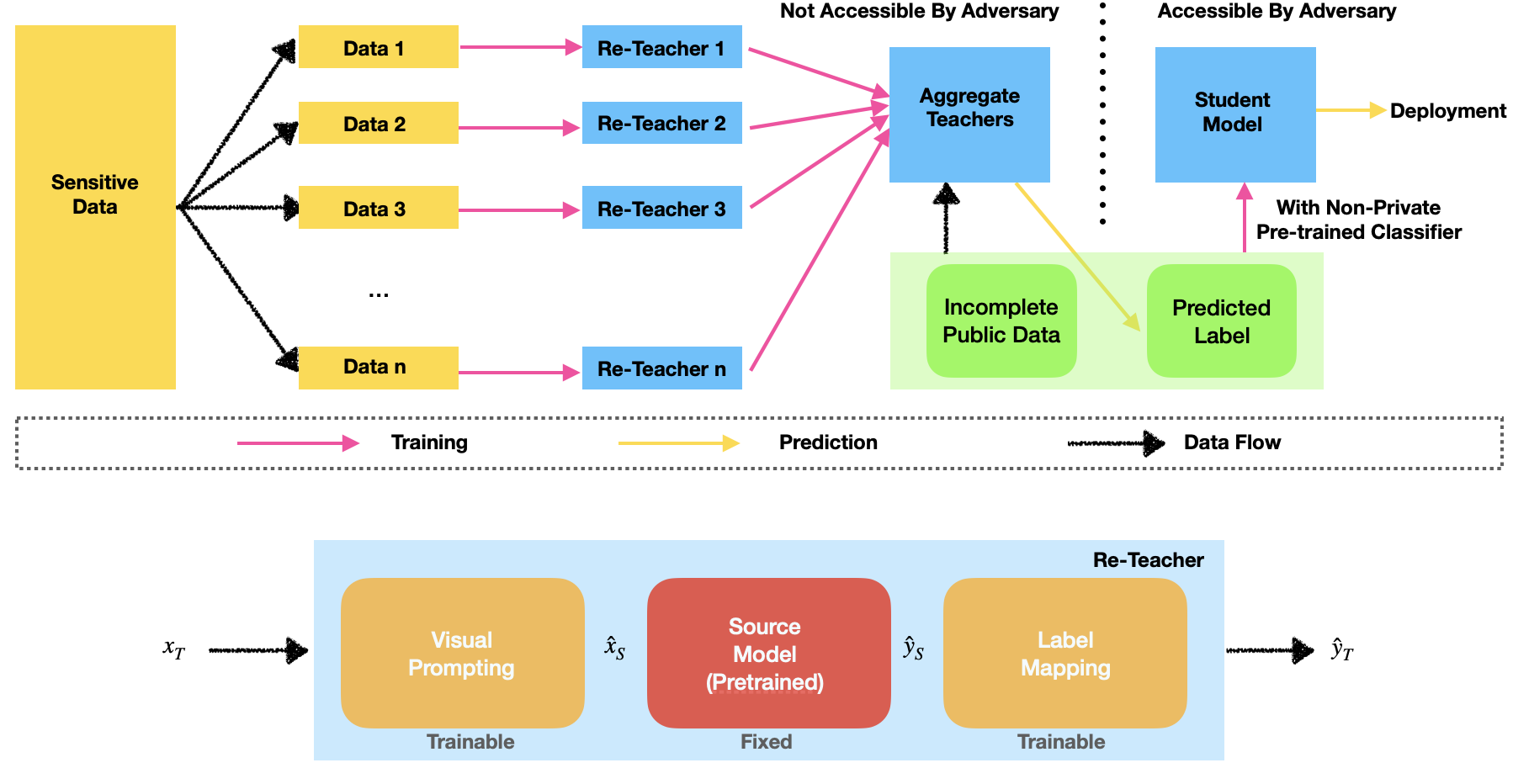}
    \caption{An overview of the proposed Prom-PATE framework. }
    \label{fig:model_architecture}
\end{figure*}

\subsection{Prom-PATE}\label{subsec:Prom-PATE}
Here, we present a new approach, Prom-PATE, which leverages VP and PATE for private learning. The workflow of Prom-PATE is shown in Figure~\ref{fig:model_architecture}. Prom-PATE is a simple yet effective 
approach to training a classifier in a DP manner. Basically, Prom-PATE follows all of the steps in PATE \cite{scalablepate, pate} except that each teacher model in PATE is reprogrammed from a pre-trained source model to a \textit{re-teacher model}. The structure of re-teacher model is also shown in Figure~\ref{fig:model_architecture}. Such simplicity of Prom-PATE also enjoys the direct inheritance of DP guarantee from PATE. 

\paragraph{Prom-PATE Procedures.}
Prom-PATE consists of three steps: (a) training re-teacher models, (b) executing private aggregation, and (c) training a student model. Step (a) considers a public pre-trained model as a \textit{source model} and trains visual prompting and label mapping on sensitive data. In particular, we are aimed to train only the prompting parameter $\omega$ while the pre-trained source model is always fixed. The prompting parameter $\omega$ (including trainable parameters $\omega_1$ and $\omega_2$ in Eq. (\ref{eq:1}) and Eq. (\ref{eq:2}), respectively) and collectively called re-teacher model (see Figure~\ref{fig:model_architecture}). We note that the re-teacher model is trained on the sensitive dataset through SGD, and hence does not fulfill DP. The next step contributes to the DP guarantee of Prom-PATE. Step (b) uses PATE to aggregate the predictions of the re-teacher models; i.e., when a sample is fed into re-teacher models, all of them have votes and use the DP noisy top-1 outcome as the label. In step (c), a student model is trained using semi-supervised learning with a pre-trained classifier. In particular, certain unlabeled public samples with labels from the DP noisy votes are used to train the student model, which serves as the resulting DP classifier. One can easily prove that Prom-PATE satisfies DP; the proof can be found in the Supplementary Materials. 

\paragraph{Training re-teacher Models.}
During the training of each re-teacher model, we keep the source model fixed while conducting SGD to update only the label mappings and visual prompts. The visual prompt $\hat{x}_{S}$ can be expressed as 
\begin{align}\label{eq:1}
    \hat{x}_{S} = M \odot \omega_{1} + (I-M) \odot \text{ZeroPad}(x_{T}),
\end{align}
where $\odot$ stands for Hadamard product, $\omega_{1}$ denotes the trainable noise parameter, and $M$ denotes the binary mask of the same dimension with the source data $x_{S}$ (i.e. $M \in \{0,1\}^{d_{S}}$, where $d_{S}$ denotes the dimension of the source domain image). On the other hand, upon obtaining the pre-trained model output $\hat{y}_{S} := f_{S}(\theta_{S}; \hat{x}_{S})$, we further render it through a label mapping function $f_{\ell}(\omega_{2};\cdot)$ that maps the source labels to target labels and obtain the final prediction $\hat{y}_{T}$ which has the following form
\begin{align}\label{eq:2}
    \hat{y}_{T} = \text{softmax}(f_{\ell}(\omega_{2}; \hat{y}_{S})).
\end{align}

\begin{figure}
    \centering
    \includegraphics[width=0.4\textwidth]{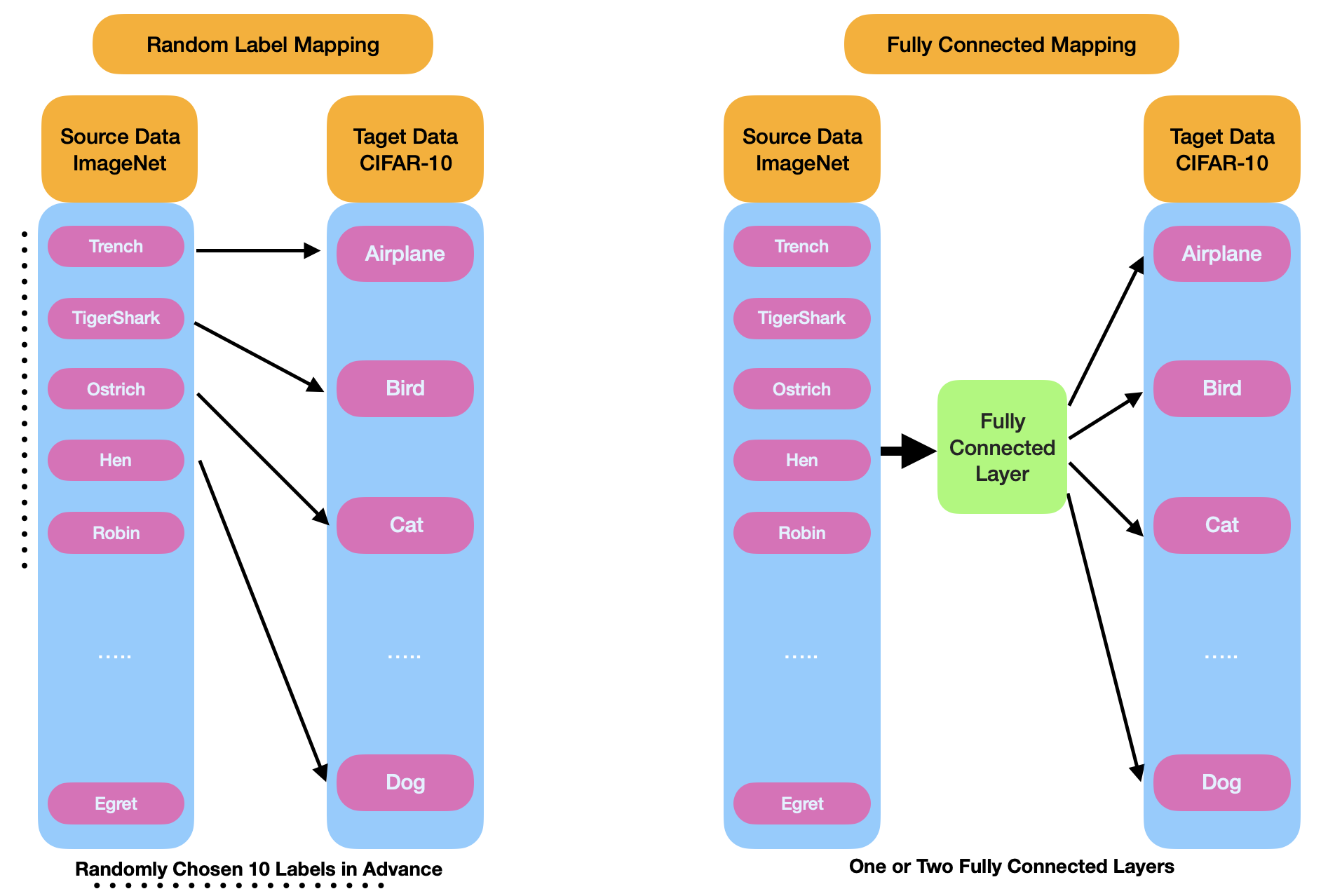}
    \caption{Illustration of different strategies for label mapping. 
    \textit{Left}: we follow the convention setting in VP \cite{bahng2022visual} and apply randomly assigned label mapping that is pre-determined before training. 
    \textit{Right}: we simply apply a trainable fully-connected layer for the model to learn the appropriate mapping as proposed in \cite{arifreprogrammable}
    }
    \label{fig:label_mapping_concept}
\end{figure}

\paragraph{Algorithmic Details of Prom-PATE.}
Figure~\ref{fig:label_mapping_concept} illustrates different label mapping techniques used in Prom-PATE. To have a correspondence in label classes between the target and source domains, in the first approach, we conduct random label mapping~\cite{bahng2022visual,tsai2020transfer}. Particularly, before training, we establish a random mapping between the labels of two domains and train the model according to the predetermined label mapping (e.g., ImageNet label $i$ $\rightarrow$ CIFAR-10 label $j$). In this case, $\omega_{2}$ specifies the source-target label correspondence in VP. For the second approach, we consider using fully connected (FC) layers as part of the label mapping for greater expressiveness, as studied in \cite{arifreprogrammable}. This allows Prom-PATE to learn how to adapt labels from the source domain to the target domain. Overall, the re-teacher models in Prom-PATE only need to train the parameters $\omega := \{\omega_{1}, \omega_{2}\}$ on the private/sensitive dataset.


To enforce DP in Prom-PATE, we adopt the DP aggregation from PATE by considering Confident-GNMax \cite{scalablepate, pate}. Specifically, given an unlabelled public data sample $x$, the aggregation mechanism would collect the response from every re-teacher model, establishing votes for each $j$-th class, $n_{j}(x)$. The aggregation then proceeds to determine whether the noisy votes are consent among re-teachers above a threshold $T$. Namely,
\begin{align}
    \max_{j}\{n_{j}(x)\} + \mathcal{N}(0, \sigma_{1}^{2}) \geq T.
\end{align}
If the inequality is met, then the aggregation would proceed to offer noisy votes of re-teachers model as follows.
\begin{align}
    \arg\max_{j}\{n_{j}(x) + \mathcal{N}(0, \sigma_{2}^{2})\}. 
\end{align}
Otherwise, the aggregation would output nothing.

To limit the privacy budget and further enhance performance in Prom-PATE, we use a subset of the public training data and label it using the private aggregation mechanism while conducting training for the rest of the training data in a semi-supervised fashion. Similar to PATE, this approach allows us to improve the privacy-utility trade-off by reducing the amount of data that needs to be labeled while still achieving high accuracy.

Since re-teacher models can adapt to the private domain under small sample complexity, we adopt the approach presented in \cite{wang2022usb} for our semi-supervised learning of the student model. We explain the details of this approach in Section \ref{subsec:compare_baseline}, where we compare it to other baseline settings. Using this approach, we can achieve a better privacy-utility trade-off and improve overall performance of Prom-PATE.


\subsection{Why Prom-PATE is Beneficial to DP?}\label{subsec:intuition}
This section provides an explanation as to why Prom-PATE, as a combination of VP and PATE, can attain an improved privacy-utility trade-off by overcoming the design challenges \textbf{(C1)}$\sim$\textbf{(C3)}.

\begin{itemize}
\item \textbf{(C1)} As mentioned in Section \ref{subsec:pate}, though PATE is superior to DPSGD from the perspectives of noise scale and information loss, it can apply only to huge datasets because, otherwise, the teacher models fail to have decent accuracy, leading to poor student classifier accuracy. However, VP has proven to successfully transfer knowledge from large source domains to small target domains~\cite{tsai2020transfer}. Thus, considering each partitioned slice of the sensitive dataset as a small target domain enables re-teacher models in Prom-PATE to avoid the problem of data insufficiency when increasing the number of re-teacher models, amplifying the benefits of ensemble learning in the ordinary PATE.

\item \textbf{(C2)} Prom-PATE is featured by utilizing the public data \textit{twice}; once in training re-teacher models and another one in training the student classifier. This can be attributed to our finding that PATE, in essence, can easily be modified to take advantage of pre-trained models (see the design of Prom-PATE in Section~\ref{subsec:Prom-PATE}). Such efficient re-use of the public data can be highly beneficial to the resulting DP classifier accuracy, as shown in Table~\ref{tab:pretrain_classifier_concept}, where Prom-PATE, Prom-PATE w/o pre-trained classifier, and PATE means utilizing public data two, one, and zero times, respectively. Obviously, the accuracy grows with the increased number of times for utilizing public data. 

\item \textbf{(C3)} Due to cross-domain capability of VP/MR~\cite{tsai2020transfer, chen2022model, yang2021voice2series}, even if the distribution of the dataset for the source model (used in training re-teacher models) is highly different from the distribution of the sensitive dataset, re-teacher models can still successfully attain high accuracies, which consequently improve the accuracy of the resulting DP classifier. Experiment evidence can be found in Section~\ref{subsec:cross-domain}.
\end{itemize}

\begin{table}[h!]\vspace{-3mm}
    \begin{center}
    \resizebox{\columnwidth}{!}{
    \begin{tabular}{|c || c c c|} 
     \hline
     \textbf{CIFAR-10} & \textbf{Prom-PATE} & \textbf{Prom-PATE w/o pre-trained classifier} & \textbf{PATE} \\ 
     \hline
     $\epsilon$ & 1.019 & 1.019  & 1.028\\ 
     \hline
     \textbf{Accuracy ± Std(\%)} & \textbf{97.07 $\pm$ 0.50} & 82.20 $\pm$ 1.14 & 32.53 $\pm$ 2.57 \\
     \hline
    \end{tabular}
    }
    \end{center}
    \caption{Effect on the pre-trained classifier.}
    \label{tab:pretrain_classifier_concept}
\end{table}




\section{Experiments}\label{sec:experiment}
In this section, we empirically evaluate the effectiveness of Prom-PATE on different datasets, with ImageNet serving as the public dataset for pre-training models. Additional experiments can be found in the Supplemental Materials. 

\subsection{Datasets and Implementation Details}\label{subsec:dataset}
We mainly use CIFAR-10 to benchmark image classification. However, we also report the results for CIFAR-100 in the Supplementary Material. 

\paragraph{Cross Domain Dataset.} To evaluate how Prom-PATE behaves in private domain adaptation with a large domain gap, we consider Blood-MNIST in our experiments. The Blood-MNIST dataset~\cite{yang2023medmnist} contains images of blood cells sampled from uninfected patients, with an original shape of $3 \times 360 \times 363$. It contains 17,092 images of 8 different blood cells (11,959 of training and 3421 of testing) and has been processed to the size of $3 \times 28 \times 28$~\cite{yang2023medmnist}. We note that the sample distribution of Blood-MNIST is highly different from the sample distribution of ImageNet because the images in Blood-MNIST are taken under microscopic devices and planar in sight. Due to the large domain gap between Blood-MNIST and ImageNet, we consider Blood-MNIST in our experiments to resolve the concern of \textbf{(C3)}. Please see Section~\ref{subsec:cross-domain} for more details.

\paragraph{Implementation Details.} All of the experiment results below are derived by averaging the results from three independent experiments. We use the official pre-trained models provided by PyTorch and set the parameters to default values for all pre-trained models. Regarding the training of each re-teacher model, since the source model is pre-trained on ImageNet, the visual prompt has a dimension of $224 \times 224$. When training the re-teacher model, we optimize the model with Adam whilst using a learning rate of $0.05$ with a decay rate of $70\%$, batch size of $16$, and training epoch of $10$. 
In Section~\ref{subsec:visual_tech}, we also investigate the effect of the binary mask $M$ on visual prompt performance.
For label mapping, we randomly select ten classes from the 1,000 source classes as a one-to-one mapping. We also use FC layers as the label mapping function in Section \ref{subsec:label_tech}. 

For the training of the student model, similar to the setting in PATE~\cite{scalablepate}, in the case of CIFAR-10, the student has access to 9,000 samples that are partially labeled through the noisy aggregation mechanism (step (b) in Prom-PATE) in Section~\ref{subsec:Prom-PATE}. The performance is evaluated on the remaining 1,000 samples in the testing set. 
Meanwhile, in the case of Blood-MNIST~\cite{yang2023medmnist}, the student has access to 2,421 samples that are as well partially labeled with privacy. The performance is evaluated on the remaining 1,000 samples in the testing set.

\paragraph{Privacy Parameter Setting.} We use R\'{e}nyi DP (RDP, see the definition in the Supplementary Materials) privacy accountant\footnote{\url{https://github.com/tensorflow/privacy/tree/master/research/pate_2018}} to calculate the privacy budget $\epsilon$. We adopt the $\delta \approx \frac{1}{n}$ convention and set $\delta = 10^{-5}$. 

\paragraph{Evaluation Metrics.} As the focus in this line of research mainly lies on image classification, we follow the convention and use the top-1 accuracy on CIFAR-10 as the metric. 


\subsection{Ablation Study of Prom-PATE}\label{subsec:compare_baseline}

We conduct an ablation study on Prom-PATE for multiple baselines that can arise from our setting. In Prom-PATE, two key components for significant improvement of accuracy are re-teacher models and the use of a pre-trained classifier in student training. Thus, there are two dimensions for the ablation study: (i) VP-based re-teacher models, transfer learning-based teacher models, and train-from-scratch teacher models and (ii) using pre-trained or train-from-scratch classifiers in semi-supervised learning of the student model. Note that these pre-trained classifiers are all trained on ImageNet. The experiment results are shown in Table~\ref{tab:compare_baseline}, where the setting A corresponds to Prom-PATE while the setting F corresponds to the ordinary PATE.

\begin{table}[htbp]
    \centering
    \small
    \resizebox{\columnwidth}{!}{
    \begin{tabular}{@{}| c | c | c |  c | c|@{}}
    \hline & Teacher & Student Training & $\epsilon$  & \shortstack{Accuracy $\pm$ Std(\%)} \\
    \hline
    \textbf{A} & VP-based re-teacher models & pre-trained & 1.019 & \textbf{97.07 $\pm$ 0.50}\\  
    \hline
    \textbf{B} & VP-based re-teacher models & train-from-scratch & 1.019 & 82.20 $\pm$ 1.14 \\  
    \hline 
    \textbf{C} & transfer learning & pre-trained & 1.021 &  96.10 $\pm$ 0.46\\ 
    \hline
    \textbf{D} & transfer learning & train-from-scratch & 1.021 & 76.93 $\pm$ 0.81 \\
    \hline
    \textbf{E} & train-from-scratch & pre-trained & 1.028 & 49.00 $\pm$ 8.97 \\
    \hline
    \textbf{F} & train-from-scratch & train-from-scratch & 1.028 & 32.53 $\pm$ 2.57 \\
    \hline
    \end{tabular}
    }
    \caption{Ablation study of Prom-PATE.}
    \label{tab:compare_baseline}
\end{table}

From Table \ref{tab:compare_baseline}, we can observe that by comparing A with C and B with D, VP-based re-teacher models in Prom-PATE indeed hold an advantage over transfer learning-based teacher models when adapting the target domain of meager data, exceeding by a maximum of 5\%. Secondly, suppose we compare A with B, C with D, and E with F, we can also see that utilizing a public pre-trained classifier in student training in Prom-PATE allows us to gain another performance improvement, ranging from 15\% to 20 \%. However, we particularly note that simply making use of a pre-trained classifier is not sufficient to have a great increase in accuracy, because the settings A and E, both containing a pre-trained classifier in the student training, have a difference of approximately 40\% in terms of the predicting accuracy. The above results support the importance of re-teacher models in Prom-PATE. Lastly, we note that albeit holding a small difference against Prom-PATE and the transfer learning baseline, we note that under a sufficient domain gap, the re-teacher tends to perform much better at these diverse private domains. We refer the readers to Section \ref{subsec:cross-domain} for more details.

\begin{table}[H]
    \begin{center}
    \footnotesize
    \begin{tabular}{|c |c | c |c |} 
     \hline
     & $\epsilon$ & sanitized $\epsilon$ & Accuracy on CIFAR-10 \\ [0.5ex] 
      \hline 
      Arif et al. \cite{arifreprogrammable} & 1.04 & 1.04 & 87.55\% \\
      \hline
      \multirow{2}{*}{Luo et al. \cite{luo2021scalable}} & 1 & 1 & 76.64\% \\ 
      & 1.5 & 1.5 & 81.57\% \\ 
      \hline
      Tramer et al. \cite{tramer2020differentially} & 2 & 2 & 92.7\% \\ 
      \hline
       \multirow{2}{*}{Yu et al. \cite{yu2022differentially}} & 1 & 1 & 94.3\% \\
      & 2 & 2 & 94.8\% \\  
      \hline 
      \multirow{2}{*}{De et al. \cite{de2022unlocking}} & 1 & 1 & 94.7\% \\ 
      & 2  & 2 & 95.4\% \\
      \hline 
      \multirow{2}{*}{Bu et al. \cite{bu2022scalable}} & 1 & 1 & 96.7\% \\ 
      & 2 & 2 & 97.1\% \\ 
      \hline
      \multirow{3}{*}{Prom-PATE} & 1.019 & 1.209 &  \textbf{99.17\%} \\ 
      & 1.505 & 1.670 & \textbf{99.07\%} \\ 
      & 1.943 & 2.250 & \textbf{99.10\%}  \\
      \hline 
    \end{tabular}
    \end{center}
    \caption{Comparison between Prom-PATE and prior work.}
    \label{tab:compare}
\end{table}

\subsection{Comparison with Existing DP Classifiers}\label{subsec:compare}
We further compare Prom-PATE against the existing work including SOTA DP classifiers. Table~\ref{tab:compare} shows the comparison results, where the accuracies of the other methods are directly excerpted from the original papers except that Yu et al.'s experiment results are from \cite{bu2022scalable}. Since Prom-PATE deploys a data-dependent bound in privacy calculation, we further follow \cite{scalablepate} to sanitize our privacy budget using smooth sensitiy analysis, preventing data leakage. The smoothed budget is marked as \textit{sanitized $\epsilon$} in Table~\ref{tab:compare}.

Table~\ref{tab:compare} shows that Prom-PATE achieves competitive performance over current existing works. In the low budget regime ($\epsilon \approx 1$), Prom-PATE outperforms all the other models and achieves the best accuracy of $99.17\%$. While the SOTA classification accuracy of CIFAR-10 (through ViT-H/14~\cite{vit}) in the non-private setting is 99.5\%\footnote{\url{https://paperswithcode.com/sota/image-classification-on-cifar-10} (last access: 2023/7)}, Prom-PATE achieves a meaningful improvement in accuracy. The reason that Prom-PATE with $\epsilon=1.019$ achieves 99.17\% in Table~\ref{tab:compare} but achieves 97.07\% in Tables~\ref{tab:baseline_concept}$\sim$\ref{tab:compare_baseline} can be attributed to our choice of implementations. In particular, the pre-trained model for re-teachers, the pre-trained model for semi-supervised learning, and the algorithm for semi-supervised learning of Prom-PATE in Table~\ref{tab:compare} are Swin Transformer~\cite{liu2021swin}, EVA~\cite{fang2023eva}, and FreeMatch~\cite{wang2022freematch}, respectively, while those of Prom-PATE in Tables~\ref{tab:baseline_concept}$\sim$\ref{tab:compare_baseline} are Swin Transformer~\cite{liu2021swin}, ViT~\cite{vit}, and FixMatch~\cite{sohn2020fixmatch}. In addition, unlike the other approaches~\cite{mixedDP, luo2021scalable, Klause2022, tramer2020differentially, de2022unlocking, bu2022scalable}, Prom-PATE enjoys great flexibility in replacing source models (in re-teacher models) by the latest classifiers and up-to-date semi-supervised training method, so as to effortlessly improve the accuracy.  

\subsection{Cross-Domain Dataset Evaluation}\label{subsec:cross-domain}
We evaluate Prom-PATE under a cross-domain setting, where the re-teacher models with public pre-trained models are visually prompted toward a small private target domain. As mentioned in Section~\ref{subsec:dataset}, we evaluate Prom-PATE on Blood-MNIST \cite{yang2023medmnist}. The experiment results are shown in Table~\ref{tab:cross-domain}, where Transfer-PATE is considered to use the same backbone source model of Prom-PATE and performs partial fine-tuning when training the teacher models.

\begin{table}[H]\vspace{-3mm}
    \begin{center}
    \resizebox{\columnwidth}{!}{
    \begin{tabular}{|c || c c c|} 
     \hline
     \textbf{Blood-MNIST} & \textbf{Prom-PATE} & \textbf{Transfer-PATE} & \textbf{Arif et al. \cite{arifreprogrammable}} \\ [0.5ex] 
     \hline
     $\epsilon$ & 1.973 & 1.983 & 1.971 \\ 
     \hline
     sanitized $\epsilon$ & 2.521 & 2.508 & 1.971 \\
     \hline
     Queries & 1000 & 1000  & - \\
     \hline
     Answered Queries & 455 & 408 & -\\ 
     \hline
     Answer Accuracy(\%) & 79.3 & 76.7 & -\\ 
     \hline 
     Threshold T & 480 & 490 & -\\  
     \hline
     $\sigma_{1}$ & 150 & 150 & -\\ 
     \hline
     $\sigma_{2}$ & 20 & 20 & - \\
     \hline 
     \textbf{Accuracy(\%)} &  \textbf{69.93} & 61.33 & 63.45\\
     \hline
    \end{tabular}
    }
    \end{center}
    \caption{Effect on cross-domain datasets.}
    \label{tab:cross-domain}
\end{table}

As one can see from Table~\ref{tab:cross-domain}, when adapting to a target domain with sufficient domain gap, Prom-PATE is able to manage the advantage of VP and maximize the accuracy gain given a fixed amount of privacy budget to vote for highly accurate labels that are beneficial for downstream student training, exceeding the Transfer-PATE by roughly 8\%. 
On the other hand, Prom-PATE is also compared against Reprogrammable-FL~\cite{arifreprogrammable}, because the latter improves accuracy in the context of FL. Prom-PATE outperforms Reprogrammable-FL by approximately 2\%. This can be attributed to much noisy perturbation of Reprogrammable-FL as stated in Section~\ref{sec:related}. Most importantly, due to the high discrepancy between ImageNet and Blood-MNIST, the high accuracy from such a train-on-ImageNet and test-on-Blood-MNIST setting also eliminates the suspicion \textbf{(C3)} from \cite{tramer2022considerations}.

\subsection{Numbers of Re-Teacher Models}\label{subsec:teacher_number}
In this section, we investigate the model performance under different numbers of re-teacher models. Table~\ref{tab:teacher_number} reports the results, where Swin Transformer~\cite{liu2021swin} is used as the source model for re-teacher models. As shown in Table \ref{tab:teacher_number}, the best utility is achieved when using 1000 re-teacher models under a privacy budget of $\epsilon \approx 1$. We also note that the accuracy of all settings with 250, 500, and 1000 re-teacher models already exceed the performance of PATE~\cite{scalablepate} under a privacy budget of $\epsilon \approx 1$.

\begin{table}[H]\vspace{-3mm}
    \begin{center}
    \resizebox{\columnwidth}{!}{
    \begin{tabular}{|c || c c c c|} 
     \hline
     \textbf{Number of re-teachers} & \textbf{100} & \textbf{250} & \textbf{500} & \textbf{1000}\\ [0.5ex] 
     \hline
     $\epsilon$ & 1.095 & 1.095 & 1.04 & 1.019 \\ 
     \hline
     Queries & 1000 & 1000 & 1000  & 1000\\
     \hline
     Answered Queries  & 18 & 46 & 90 & 684   \\
     \hline
     Threshold T & 430 & 500 & 650 & 500\\  
     \hline
     $\sigma_{1}$ & 150 & 150 & 150 & 200 \\ 
     \hline
     $\sigma_{2}$ & 50 & 100 & 100 & 50 \\
     \hline 
     \textbf{Accuracy\%) ± Std} & 59.20 $\pm$ 0 & 85.87 $\pm$ 0.55 &96.53 $\pm$ 0.74 & \textbf{97.07 $\pm$ 0.50} \\
     \hline
    \end{tabular}
    }
    \end{center}
    \caption{Effect on different numbers of re-teacher models.}
    \label{tab:teacher_number}
\end{table}

\subsection{Different Pre-Trained Models}\label{subsec:pretrain_type}
We study the effect of different pre-trained source models on Prom-PATE. Table~\ref{tab:pretrained} reports the results. In particular, using Swin Transformer \cite{liu2021swin} as the pre-trained source model results in the best performance of 99\% on CIFAR-10. This is consistent with the theoretical relationship presented in \cite{yang2021voice2series}, which states that the population risk on the target task of the reprogrammed model can be upper bounded by the source risk with an additional term in misalignment error. Therefore, as we can see from Table \ref{tab:pretrained_source}, which includes the accuracy of pre-trained models on the source domain (i.e., source risk), Swin Transformer has the least empirical risk and serves as a natural choice for the source model.
\begin{table*}[hbt!]
    \centering
    \resizebox{2\columnwidth}{!}{
    \begin{tabular}{@{}|c|c|cccccc||c|@{}}
    \hline & $\epsilon$   & \shortstack{Queries} & \shortstack{Answered Queries} & \shortstack{Answered Accuracy(\%)} & \shortstack{Threshold $T$} & \shortstack{$\sigma_{1}$} & \shortstack{$\sigma_{2}$} & \shortstack{Accuracy $\pm$ Std(\%)} \\ 
    \hline
    \textbf{ResNet50} & 1.081 & 1000 & 461 & 91.3 & 650 & 200 & 50 & 95.27 $\pm$ 0.80\\
    \hline 
    \textbf{ResNet152} & 1.009 & 1000 & 604 & 93.9 & 620 & 200 & 50 & 95.40 $\pm$ 0.40  \\ 
    \hline
    \textbf{WideResNet} & 1.068 & 1000 & 555 & 90.8 & 620 & 200 & 50 & 94.37 $\pm$ 0.25\\ 
    \hline
    \textbf{ViT} & 1.007 & 1000 & 660 & 93.6 & 600 & 200 & 50 & 95.53 $\pm$ 0.51 \\ 
    \hline
    \textbf{Swin} & 1.019 & 1000 & 684 & 94.7 & 600 & 200 & 50 & \textbf{97.07 $\pm$ 0.50} \\
    \hline
    \end{tabular}
    }
    \caption{Effect on different pre-trained models.}
    \label{tab:pretrained}
\end{table*}

\begin{table}[H]\vspace{-2mm}
    \begin{center}
    \footnotesize
    \begin{tabular}{|c || c|} 
     \hline
     \textbf{ImageNet} & \textbf{Accuracy} \\ [0.5ex] 
     \hline
     \textbf{ResNet50} & 79.3\\
     \hline
     \textbf{ResNet152} & 78.5\\
     \hline
     \textbf{WideResNet} & 78.1 \\  
     \hline
     \textbf{ViT} &  84.0 \\ 
     \hline
     \textbf{Swin Transformer} & 85.2 \\
     \hline 
    \end{tabular}
    \end{center}
    \caption{Test accuracy of ImageNet source models.}
    \label{tab:pretrained_source}
\end{table}

\subsection{Binary Mask in Visual Prompting}\label{subsec:visual_tech}
We further study how the different visual prompting techniques affect classification accuracy. Specifically, we consider two settings on whether to apply the binary mask $M$ or not. Table~\ref{tab:prompt_technique} reports the results, where Swin transformer~\cite{liu2021swin} as the source model with $1000$ re-teacher models is considered.

\begin{table}[H]
    \begin{center}
    \resizebox{\columnwidth}{!}{
    \begin{tabular}{|c || c c|} 
     \hline
     \textbf{Prompting Technique} & \textbf{Without Mask $M$} & \textbf{With Mask $M$} \\ [0.5ex] 
     \hline
     $\epsilon$ & 1.017 & 1.019\\ 
     \hline
     Queries & 1000 & 1000 \\
     \hline
     Answered Queries & 675 & 684 \\ 
     \hline
     Answer Accuracy(\%) & 94.8 & 94.7 \\
     \hline
     Threshold T & 600 & 600 \\  
     \hline
     $\sigma_{1}$ & 200 & 200 \\ 
     \hline
     $\sigma_{2}$ & 50 & 50 \\
     \hline 
     \textbf{Accuracy ± Std(\%)} & 96.53 $\pm$ 0.32 & \textbf{97.07 $\pm$ 0.50 } \\
     \hline
    \end{tabular}
    }
    \end{center}
    \caption{Effect on visual prompting technique}
    \label{tab:prompt_technique}
\end{table}\vspace{-4mm}

One can observe from Table~\ref{tab:prompt_technique} that using $M$ could enhance performance. The rationale is that by utilizing $M$, we can control the amount of noise placed in the visual prompt, hence controlling the ratio of target data $x_{T}$ and noise parameter $\omega_{1}$. This leads to a better trade-off between accuracy and the meager amount of private data each re-teacher model owns.

\subsection{Label Mapping Techniques}\label{subsec:label_tech}
Next, we proceed to investigate the effect of label mapping on Prom-PATE. Particularly, we consider the settings of using random label mapping (RLM), one fully-connected layer, and two fully-connected layers (see Figure~\ref{fig:label_mapping_concept}).
Table~\ref{tab:mapping_technique} shows the experiment results, where Swin transformer~\cite{liu2021swin} as the source model with $1000$ re-teacher models is considered. In particular, using one FC layer allows Prom-PATE to achieve the best performance. Furthermore, we note that randomly selecting ten classes for mapping would disrupt the behavior of the pre-trained model, as the mapping relations among source and target labels are randomly given but other remaining source classes might contain valuable information for the prediction. Such an explanation can be confirmed by the accuracy (i.e., noisy label accuracy) of RLM, which is only $22.9\%$, demonstrating that even with a high consensus of the re-teacher models, the ensemble prediction is likely to be wrong as well. On the other hand, while using two FC layers allows for more expressiveness, the number of training parameters is increased as well, leading to a slight degradation in accuracy with limited training data for each re-teacher model.

\begin{table}[hbt]\vspace{-1mm}
    \begin{center}
    \resizebox{\columnwidth}{!}{
    \begin{tabular}{|c || c c c|} 
     \hline
     \textbf{Mapping Technique} & \textbf{RLM} & \textbf{1-Layer FC} & \textbf{2-Layer FC} \\ [0.5ex] 
     \hline
     $\epsilon$ & 1.042 & 1.019 & 1.026 \\ 
     \hline
     Queries & 1000 & 1000  & 1000 \\
     \hline
     Answered Queries & 109 & 684 & 336\\ 
     \hline
     Answer Accuracy(\%) & 22.9 & 94.7 & 92.6\\ 
     \hline 
     Threshold T & 650 & 600 & 670\\  
     \hline
     $\sigma_{1}$ & 200 & 200 & 200\\ 
     \hline
     $\sigma_{2}$ & 50 & 50 & 50\\
     \hline 
     \textbf{Accuracy ± Std(\%)} & 33.4 $\pm$ 0.66 & \textbf{97.07 $\pm$ 0.50 } & 96.13 $\pm$ 0.41\\
     \hline
    \end{tabular}
    }
    \end{center}
    \caption{Effect on label mapping techniques.}
    \label{tab:mapping_technique}
\end{table}

\subsection{Rescale Ratio in Visual Prompting}\label{subsec:visual_tradeoff}
Usually, in VP/MR, the image from the target domain needs to be rescaled and surrounded by trainable noises, as shown in Eq. (\ref{eq:1}). The resulting $\hat{x}_S$ can then be fed into the source model. A higher rescale ratio generally leads to better performance. The rationale is that a higher rescale ratio provides more information from the target domain, which enables the re-teacher model to generate better visual prompts that can more effectively guide the source model in learning the relevant features of the target domain. However, a too-high rescale ratio could potentially result in overfitting to the target domain, leading to poor generalization performance. Hence, one strikes a balance between providing sufficient information from the target domain and avoiding overfitting. In our experiments, a rescale ratio of $0.6$ achieves the best performance.

\begin{table}[hbt]\vspace{-1mm}
    \centering
    \small
    \resizebox{\columnwidth}{!}{
    \begin{tabular}{@{}|c|l|ccccc||c|@{}}
    \hline \textbf{Rescale Size} & \hspace{2ex}$\epsilon$   & \shortstack{AQ} & \shortstack{AA(\%)} & \shortstack{$T$} & \shortstack{$\sigma_{1}$} & \shortstack{$\sigma_{2}$} & \shortstack{Accuracy $\pm$ Std(\%)} \\ 
    \hline
    \textbf{$64\times64$} & 1.028 & 408 & 86.3 &650 & 200 & 50 & 93.03 $\pm$ 1.0 \\ 
    \hline
    \textbf{$128\times128$} & 1.016 & 662 & 92.6 &610 & 200 & 50 & 95.83 $\pm$ 0.1 \\ 
    \hline
    \textbf{$160\times160$} & 1.016 & 655 & 93.7 &610 & 200 & 50 & 95.07 $\pm$ 0.3 \\ 
    \hline
    \textbf{$192\times192$} & 1.019 & 684 & 94.7 &600 & 200 & 50 & \textbf{97.07 $\pm$ 0.5} \\
    \hline
    \textbf{$210\times210$} & 1.016 & 655 & 93.7 &610 & 200 & 50 & 95.30 $\pm$ 0.5 \\
    \hline
    \end{tabular}
    }
    \caption{Effect on the rescale ratio of target Data. The number of queries is 1,000. AQ, AA, and T denote answered queries, answered accuracy (\%), and threshold, respectively.}
    \label{tab:rescale}
\end{table}

As observed from Table \ref{tab:rescale}, rescaling $x_T$ to $192 \times 192$ for visual prompting achieved the highest utility. As explained in Section \ref{subsec:visual_tech}, the rescale size provides a ratio between the trainable parameter $\omega_1$ and target data $x_T$. Too many noise parameters and a small target image might degrade performance due to the quality of the target image and insufficient data. Conversely, a larger target image and fewer parameters of $\omega_1$ might cause sub-optimal input transformation from target to source, leading to a poor prompt.

\section{Conclusion}\label{sec:conclusion}
In this paper, we conducted a comprehensive study and discovered a new benefit of VP in DP. In particular, we propose Prom-PATE, a new VP-empowered training method for constructing DP classifiers. Prom-PATE leverages VP to assist in the adaptation of pre-trained models in a more efficient way without losing privacy. Empirical evaluations show that Prom-PATE provides SOTA performance compared to several baselines and existing works. We also find that Prom-PATE achieves an even better accuracy gain when the target task has a sufficient domain gap against the pre-trained model (i.e., the ImageNet to Blood-MNIST setting), demonstrating the generality of Prom-PATE.
Our findings
suggest that VP is a promising approach to facilitating further research in building DP classifiers that improve or even extinguish the privacy-utility trade-off.

\clearpage
{\small
\bibliographystyle{ieee_fullname}
\bibliography{egbib}

\begin{thebibliography}{10}\itemsep=-1pt

\bibitem{abadi2016deep}
Martin Abadi, Andy Chu, Ian Goodfellow, H.~Brendan McMahan, Ilya Mironov, Kunal
  Talwar, and Li Zhang.
\newblock Deep learning with differential privacy.
\newblock In {\em ACM Conference on Computer and Communications Security
  (CCS)}, 2016.

\bibitem{arifreprogrammable}
Huzaifa Arif, Alex Gittens, and Pin-Yu Chen.
\newblock Reprogrammable-fl: Improving utility-privacy tradeoff in federated
  learning via model reprogramming.
\newblock In {\em First IEEE Conference on Secure and Trustworthy Machine
  Learning}, 2023.

\bibitem{bahng2022visual}
Hyojin Bahng, Ali Jahanian, Swami Sankaranarayanan, and Phillip Isola.
\newblock Visual prompting: Modifying pixel space to adapt pre-trained models.
\newblock {\em arXiv preprint arXiv:2203.17274}, 2022.

\bibitem{bar2022visual}
Amir Bar, Yossi Gandelsman, Trevor Darrell, Amir Globerson, and Alexei~A Efros.
\newblock Visual prompting via image inpainting.
\newblock In {\em Conference on Neural Information Processing Systems
  (NeurIPS)}, 2022.

\bibitem{bu2022scalable}
Zhiqi Bu, Jialin Mao, and Shiyun Xu.
\newblock Scalable and efficient training of large convolutional neural
  networks with differential privacy.
\newblock In {\em Advances in Neural Information Processing Systems (NeurIPS)},
  2022.

\bibitem{chen2022visual}
Aochuan Chen, Peter Lorenz, Yuguang Yao, Pin-Yu Chen, and Sijia Liu.
\newblock Visual prompting for adversarial robustness.
\newblock {\em arXiv preprint arXiv:2210.06284}, 2022.

\bibitem{chen2023understanding}
Aochuan Chen, Yuguang Yao, Pin-Yu Chen, Yihua Zhang, and Sijia Liu.
\newblock Understanding and improving visual prompting: A label-mapping
  perspective.
\newblock In {\em Proceedings of the IEEE conference on computer vision and
  pattern recognition}, 2023.

\bibitem{chen2022model}
Pin-Yu Chen.
\newblock Model reprogramming: Resource-efficient cross-domain machine
  learning.
\newblock {\em arXiv preprint arXiv:2202.10629}, 2022.

\bibitem{simclr}
Ting Chen, Simon Kornblith, Mohammad Norouzi, and Geoffrey Hinton.
\newblock A simple framework for contrastive learning of visual
  representations.
\newblock In {\em International Conference on Machine Learning (ICML)}, 2020.

\bibitem{Chourasia2021DifferentialPD}
R. Chourasia, Jiayuan Ye, and R. Shokri.
\newblock Differential privacy dynamics of langevin diffusion and noisy
  gradient descent.
\newblock {\em Conference on Neural Information Processing Systems (NeurIPS)},
  2021.

\bibitem{de2022unlocking}
Soham De, Leonard Berrada, Jamie Hayes, Samuel~L Smith, and Borja Balle.
\newblock Unlocking high-accuracy differentially private image classification
  through scale.
\newblock {\em arXiv preprint arXiv:2204.13650}, 2022.

\bibitem{ding2023parameter}
Ning Ding, Yujia Qin, Guang Yang, Fuchao Wei, Zonghan Yang, Yusheng Su,
  Shengding Hu, Yulin Chen, Chi-Min Chan, Weize Chen, et~al.
\newblock Parameter-efficient fine-tuning of large-scale pre-trained language
  models.
\newblock {\em Nature Machine Intelligence}, 5(3):220--235, 2023.

\bibitem{vit}
Alexey Dosovitskiy, Lucas Beyer, Alexander Kolesnikov, Dirk Weissenborn,
  Xiaohua Zhai, Thomas Unterthiner, Mostafa Dehghani, Matthias Minderer, Georg
  Heigold, Sylvain Gelly, Jakob Uszkoreit, and Neil Houlsby.
\newblock An image is worth 16x16 words: Transformers for image recognition at
  scale.
\newblock In {\em International Conference on Learning Representations (ICLR)},
  2021.

\bibitem{elsayed2018adversarial}
Gamaleldin~F. Elsayed, Ian Goodfellow, and Jascha Sohl-Dickstein.
\newblock Adversarial reprogramming of neural networks.
\newblock In {\em International Conference on Learning Representations}, 2019.

\bibitem{fang2023eva}
Yuxin Fang, Wen Wang, Binhui Xie, Quan Sun, Ledell Wu, Xinggang Wang, Tiejun
  Huang, Xinlong Wang, and Yue Cao.
\newblock Eva: Exploring the limits of masked visual representation learning at
  scale.
\newblock In {\em Proceedings of the IEEE/CVF Conference on Computer Vision and
  Pattern Recognition}, pages 19358--19369, 2023.

\bibitem{mixedDP}
Aditya Golatkar, Alessandro Achille, Yu-Xiang Wang, Aaron Roth, Michael Kearns,
  and Stefano Soatto.
\newblock Mixed differential privacy in computer vision.
\newblock In {\em The IEEE/CVF Computer Vision and Pattern Recognition
  Conference (CVPR)}, 2022.

\bibitem{helber2019eurosat}
Patrick Helber, Benjamin Bischke, Andreas Dengel, and Damian Borth.
\newblock Eurosat: A novel dataset and deep learning benchmark for land use and
  land cover classification.
\newblock {\em IEEE Journal of Selected Topics in Applied Earth Observations
  and Remote Sensing}, 12(7):2217--2226, 2019.

\bibitem{jia2022visual}
Menglin Jia, Luming Tang, Bor-Chun Chen, Claire Cardie, Serge Belongie, Bharath
  Hariharan, and Ser-Nam Lim.
\newblock Visual prompt tuning.
\newblock In {\em Computer Vision--ECCV 2022: 17th European Conference, Tel
  Aviv, Israel, October 23--27, 2022, Proceedings, Part XXXIII}, pages
  709--727. Springer, 2022.

\bibitem{kaplan2020scaling}
Jared Kaplan, Sam McCandlish, Tom Henighan, Tom~B Brown, Benjamin Chess, Rewon
  Child, Scott Gray, Alec Radford, Jeffrey Wu, and Dario Amodei.
\newblock Scaling laws for neural language models.
\newblock {\em arXiv preprint arXiv:2001.08361}, 2020.

\bibitem{karras2019style}
Tero Karras, Samuli Laine, and Timo Aila.
\newblock A style-based generator architecture for generative adversarial
  networks.
\newblock In {\em Proceedings of the IEEE/CVF conference on computer vision and
  pattern recognition}, pages 4401--4410, 2019.

\bibitem{Klause2022}
Helena Klause, Alexander Ziller, Daniel Rueckert, Kerstin Hammernik, and
  Georgios Kaissis.
\newblock Differentially private training of residual networks with scale
  normalisation.
\newblock In {\em Theory and Practice of Differential Privacy (TPDP)}, 2022.

\bibitem{krizhevsky2009learning}
Alex Krizhevsky, Geoffrey Hinton, et~al.
\newblock Learning multiple layers of features from tiny images.
\newblock 2009.

\bibitem{Large-Language-Models}
Xuechen Li, Florian Tramèr, Percy Liang, and Tatsunori Hashimoto.
\newblock Large language models can be strong differentially private learners.
\newblock In {\em International Conference on Learning Representations (ICLR)},
  2022.

\bibitem{liu2023pre}
Pengfei Liu, Weizhe Yuan, Jinlan Fu, Zhengbao Jiang, Hiroaki Hayashi, and
  Graham Neubig.
\newblock Pre-train, prompt, and predict: A systematic survey of prompting
  methods in natural language processing.
\newblock {\em ACM Computing Surveys}, 55(9):1--35, 2023.

\bibitem{liu2021swin}
Ze Liu, Yutong Lin, Yue Cao, Han Hu, Yixuan Wei, Zheng Zhang, Stephen Lin, and
  Baining Guo.
\newblock Swin transformer: Hierarchical vision transformer using shifted
  windows.
\newblock In {\em Proceedings of the IEEE/CVF international conference on
  computer vision}, pages 10012--10022, 2021.

\bibitem{liu2018large}
Ziwei Liu, Ping Luo, Xiaogang Wang, and Xiaoou Tang.
\newblock Large-scale celebfaces attributes (celeba) dataset.
\newblock {\em Retrieved August}, 15(2018):11, 2018.

\bibitem{luo2021scalable}
Zelun Luo, Daniel~J Wu, Ehsan Adeli, and Li Fei-Fei.
\newblock Scalable differential privacy with sparse network finetuning.
\newblock In {\em IEEE/CVF Conference on Computer Vision and Pattern
  Recognition}, pages 5059--5068, 2021.

\bibitem{RDP}
Ilya Mironov.
\newblock R{\'e}nyi differential privacy.
\newblock {\em 2017 IEEE 30th Computer Security Foundations Symposium (CSF)},
  pages 263--275, 2017.

\bibitem{mironov2017renyi}
Ilya Mironov.
\newblock R{\'e}nyi differential privacy.
\newblock In {\em 2017 IEEE 30th computer security foundations symposium
  (CSF)}, pages 263--275. IEEE, 2017.

\bibitem{Nasr2021AdversaryIL}
Milad Nasr, Shuang Songi, Abhradeep Thakurta, Nicolas Papernot, and Nicholas
  Carlin.
\newblock Adversary instantiation: Lower bounds for differentially private
  machine learning.
\newblock In {\em 2021 IEEE Symposium on Security and Privacy (SP)}, 2021.

\bibitem{neekhara2022cross}
Paarth Neekhara, Shehzeen Hussain, Jinglong Du, Shlomo Dubnov, Farinaz
  Koushanfar, and Julian McAuley.
\newblock Cross-modal adversarial reprogramming.
\newblock In {\em Proceedings of the IEEE/CVF Winter Conference on Applications
  of Computer Vision}, pages 2427--2435, 2022.

\bibitem{netzer2011reading}
Yuval Netzer, Tao Wang, Adam Coates, Alessandro Bissacco, Bo Wu, and Andrew~Y.
  Ng.
\newblock Reading digits in natural images with unsupervised feature learning.
\newblock {\em NIPS Workshop on Deep Learning and Unsupervised Feature
  Learning}, 2011.

\bibitem{pate}
Nicolas Papernot, Martín Abadi, Úlfar Erlingsson, Ian Goodfellow, and Kunal
  Talwar.
\newblock Semi-supervised knowledge transfer for deep learning from private
  training data.
\newblock In {\em International Conference on Learning Representations (ICLR)},
  2017.

\bibitem{scalablepate}
Nicolas Papernot, Shuang Song, Ilya Mironov, Ananth Raghunathan, Kunal Talwar,
  and {\'U}lfar Erlingsson.
\newblock Scalable private learning with pate.
\newblock In {\em International Conference on Learning Representations (ICLR)},
  2018.

\bibitem{radford2021learning}
Alec Radford, Jong~Wook Kim, Chris Hallacy, Aditya Ramesh, Gabriel Goh,
  Sandhini Agarwal, Girish Sastry, Amanda Askell, Pamela Mishkin, Jack Clark,
  et~al.
\newblock Learning transferable visual models from natural language
  supervision.
\newblock In {\em International conference on machine learning}, pages
  8748--8763. PMLR, 2021.

\bibitem{sohn2020fixmatch}
Kihyuk Sohn, David Berthelot, Nicholas Carlini, Zizhao Zhang, Han Zhang,
  Colin~A Raffel, Ekin~Dogus Cubuk, Alexey Kurakin, and Chun-Liang Li.
\newblock Fixmatch: Simplifying semi-supervised learning with consistency and
  confidence.
\newblock {\em Advances in neural information processing systems}, 33:596--608,
  2020.

\bibitem{tramer2020differentially}
Florian Tram{\`e}r and Dan Boneh.
\newblock Differentially private learning needs better features (or much more
  data).
\newblock In {\em International Conference on Learning Representations (ICLR)},
  2021.

\bibitem{tramer2022considerations}
Florian Tram{\`e}r, Gautam Kamath, and Nicholas Carlini.
\newblock Considerations for differentially private learning with large-scale
  public pretraining.
\newblock {\em arXiv:2212.06470}, 2022.

\bibitem{tsai2020transfer}
Yun-Yun Tsai, Pin-Yu Chen, and Tsung-Yi Ho.
\newblock Transfer learning without knowing: Reprogramming black-box machine
  learning models with scarce data and limited resources.
\newblock In {\em International Conference on Machine Learning (ICML)}, 2020.

\bibitem{wang2022usb}
Yidong Wang, Hao Chen, Yue Fan, Wang Sun, Ran Tao, Wenxin Hou, Renjie Wang,
  Linyi Yang, Zhi Zhou, Lan-Zhe Guo, et~al.
\newblock Usb: A unified semi-supervised learning benchmark.
\newblock In {\em Conference on Neural Information Processing Systems
  (NeurIPS)}, 2022.

\bibitem{wang2022freematch}
Yidong Wang, Hao Chen, Qiang Heng, Wenxin Hou, Yue Fan, Zhen Wu, Jindong Wang,
  Marios Savvides, Takahiro Shinozaki, Bhiksha Raj, et~al.
\newblock Freematch: Self-adaptive thresholding for semi-supervised learning.
\newblock {\em arXiv preprint arXiv:2205.07246}, 2022.

\bibitem{yang2021voice2series}
Chao-Han~Huck Yang, Yun-Yun Tsai, and Pin-Yu Chen.
\newblock Voice2series: Reprogramming acoustic models for time series
  classification.
\newblock In {\em International Conference on Machine Learning}, pages
  11808--11819. PMLR, 2021.

\bibitem{yang2023medmnist}
Jiancheng Yang, Rui Shi, Donglai Wei, Zequan Liu, Lin Zhao, Bilian Ke,
  Hanspeter Pfister, and Bingbing Ni.
\newblock Medmnist v2-a large-scale lightweight benchmark for 2d and 3d
  biomedical image classification.
\newblock {\em Scientific Data}, 10(1):41, 2023.

\bibitem{yu2022differentially}
Da Yu, Saurabh Naik, Arturs Backurs, Sivakanth Gopi, Huseyin~A Inan, Gautam
  Kamath, Janardhan Kulkarni, Yin~Tat Lee, Andre Manoel, Lukas Wutschitz,
  Sergey Yekhanin, and Huishuai Zhang.
\newblock Differentially private fine-tuning of language models.
\newblock In {\em International Conference on Learning Representations (ICLR)},
  2022.

\bibitem{Large-scale-private-learning}
Da Yu, Huishuai Zhang, Wei Chen, Jian Yin, and Tie-Yan Liu.
\newblock Large scale private learning via low-rank reparametrization.
\newblock In {\em International Conference on Machine Learning (ICML)}, 2021.

\bibitem{private-knn}
Yuqing Zhu, Xiang Yu, Manmohan Chandraker, and Yu-Xiang Wang.
\newblock Private-knn: Practical differential privacy for computer vision.
\newblock In {\em Proceedings of the IEEE/CVF Conference on Computer Vision and
  Pattern Recognition}, pages 11854--11862, 2020.

\end{thebibliography}
}

\clearpage
\section{Supplementary Material}
The supplementary material has two parts; the privacy analysis of Prom-PATE in Section~\ref{sec: Privacy Analysis of Prom-PATE} and the additional experiment results in Section~\ref{sec: Additional Experiments}.

\subsection{Privacy Analysis of Prom-PATE}\label{sec: Privacy Analysis of Prom-PATE}
Here, we first present the definitions for $(\epsilon,\delta)$-differential privacy ($(\epsilon,\delta)$-DP) and its variant, R\`{e}nyi Differential Privacy (RDP). After that, we prove that Prom-PATE satisfies $(\epsilon,\delta)$-DP.

\begin{definition}[Differential Privacy, DP] A randomized algorithm $\mathcal{M}$ is $(\varepsilon, \delta)$-DP if for all $\mathcal{S}\subseteq \text{Range}(\mathcal{M})$ and for any neighboring datasets $\mathcal{D}$ and $\mathcal{D}'$,
\begin{align}\label{eq: DP}
\text{Pr}[\mathcal{M}(\mathcal{D})\in \mathcal{S}]\leq e^{\varepsilon} \text{Pr}[\mathcal{M}(\mathcal{D}')\in \mathcal{S}]+\delta.
\end{align}\label{def: DP}
\end{definition}

In Definition~\ref{def: DP}, $\mathcal{D}$ and $\mathcal{D}'$ are neighboring if $\mathcal{D}$ can be obtained by adding or removing one sample from $\mathcal{D}'$

\begin{definition} [R\`{e}nyi Differential Privacy, RDP]
\label{def: RDP}
A randomized algorithm $\mathcal{M}$ is $(\alpha, \epsilon (\alpha))$-RDP with $\alpha>1$ if for any neighboring datasets $\mathcal{D}$ and $\mathcal{D}'$,
\begin{scriptsize}
\begin{align}
\label{eq: RDP}
D_\alpha(\mathcal{M}(\mathcal{D})||\mathcal{M}(\mathcal{D}'))=\frac{1}{\alpha-1}\log \mathbb{E}_{x \thicksim\mathcal{M}(\mathcal{D}')}\left[ \left( \frac{\text{Pr}[\mathcal(\mathcal{D})=x]}{\text{Pr}[\mathcal(\mathcal{D}')=x]} \right)^{\alpha-1}\right]\leq \epsilon (\alpha),
\end{align}
\end{scriptsize}
where $D_{\alpha}$ is the R\'{e}nyi divergence of order $\alpha$.
\end{definition}

\begin{theorem}[Sequential Composition on RDP~\cite{RDP}]
If the mechanism $\mathcal{M}_1$ satisfies $(\alpha, \epsilon_{1})$-RDP and the mechanism $\mathcal{M}_2$ satisfies $(\alpha, \epsilon_{2})$-RDP, then $\mathcal{M}_2\circ \mathcal{M}_1$ satisfies $(\alpha, \epsilon_{1} + \epsilon_{2})$-RDP.
\label{thm: gaussian composition in RDP}
\end{theorem}

\begin{theorem}[Translation of RDP~\cite{mironov2017renyi}]
If a mechanism $\mathcal{M}$ satisfies $(\alpha , \epsilon)$-RDP, then $\mathcal{M}$ satisfies $(\epsilon + \frac{\log(1/\delta)}{\alpha -1 }, \delta)$-DP for any $\delta\in (0, 1)$.
\label{thm: translation}
\end{theorem}

With the above definition and results, in Theorem~\ref{thm: privacy proof of Prom-PATE} we prove that Prom-PATE satisfies $(\epsilon, \delta)$-DP.

\begin{theorem}
Prom-PATE satisfies $(\epsilon, \delta)$-DP.
\label{thm: privacy proof of Prom-PATE}\end{theorem}
\begin{proof}
Basically, Prom-PATE follows the framework of PATE. One can find that the re-teacher models will not be released according to the design of PATE. Therefore, the only step that Prom-PATE ``touches'' the sensitive dataset is when the student model queries the re-teacher ensemble for labeling the unlabelled samples. Hence, by leveraging proposition $8$ in \cite{scalablepate}, we know that the Confident-GNMax fulfills $(\alpha, \frac{\alpha}{2\sigma^{2}})$-RDP guarantee. Here, consider the case where Prom-PATE has $T$ unlabeled samples that need to be labeled by the Confident-GNMax result before the training of the student model. According to Theorem~\ref{thm: gaussian composition in RDP}, Prom-PATE satisfies $(\alpha, T\cdot \frac{\alpha}{2\sigma^{2}})$-RDP. Hence, equivalently, Prom-PATE satisfies $(\epsilon, \delta)$-DP for any $\delta\in (0, 1)$, where $\epsilon=T\cdot \frac{\alpha}{2\sigma^{2}}+ \frac{\log(1/\delta)}{\alpha -1 }$ according to Theorem~\ref{thm: translation}.
\end{proof}

We convert the current data-independent proof to the data-dependent proof. However, the data-dependent proof for Prom-PATE is identical to Proposition 7 and Theorem 6 in \cite{scalablepate}  due to the design of Prom-PATE. Therefore, we skip the proof here. 

\subsection{Additional Experiments}\label{sec: Additional Experiments}
We start by presenting the characteristics and settings of the datasets used in our experiments in Table~\ref{tab:dataset-intro}. Then, we present additional experiment results for the cross-domain tasks and high-resolution images. 

\begin{table*}[hbt!]
    \centering
    \resizebox{2\columnwidth}{!}{
    \begin{tabular}{@{}|c|cccc||c|@{}}
    \hline   & \shortstack{Channel No.} & \shortstack{Class No.} & \shortstack{Re-Teacher Models Training Data Size} & \shortstack{Student Model Training Data Size} &  \shortstack{Testing Data Size} \\ 
    \hline
    \textbf{SVHN \cite{netzer2011reading}} & 3 & 10 & 73,257 & 25,032 & 1,000 \\
    \hline 
    \textbf{EuroSAT \cite{helber2019eurosat}}  & 3 & 10 & 10,000 & 16,000 & 1,000  \\ 
    \hline
    \textbf{PathMNIST \cite{yang2023medmnist}}  & 3 & 9 & 89,996 & 6,180 & 1,000 \\ 
    \hline
    \textbf{TissueMNIST \cite{yang2023medmnist}} & 1 & 8 & 165,466 & 46,280 & 1,000  \\ 
    \hline
    \textbf{DermaMNIST \cite{yang2023medmnist}}& 3 & 7 & 7,007 & 1,505 & 500 \\
    \hline
    \textbf{CelebA-Gender \cite{liu2018large}} & 3 & 2 & 162,770 & 18,962 & 1,000 \\ 
    \hline 
    \textbf{CelebA-Hair \cite{liu2018large}} & 3 & 3 & 108,358 & 10,668 & 1,000 \\ 
    \hline 
    \textbf{FFHQ-Gender \cite{karras2019style}} & 3 & 2 & 50,000 & 18,471 & 1,000 \\ 
    \hline
    \textbf{CIFAR-100 \cite{krizhevsky2009learning}} & 3 & 100 & 50,000 & 9,000 & 1,000 \\ 
    \hline
    \end{tabular}
    }
    \caption{Characteristics and experiment settings of different datasets.}
    \label{tab:dataset-intro}
\end{table*}

\paragraph{Additional Experiment Results for Cross-Domain Tasks.} As the source model is trained on ImageNet, we evaluate in Section~\ref{subsec:cross-domain} the performance of Prom-PATE on Blood-MNIST to demonstrate the superiority of Prom-PATE under a cross-domain setting. Here, we provide more experiment results (Tables~\ref{tab:SVHN}$\sim$~\ref{tab:DermaMNIST}) on different datasets. In Tables~\ref{tab:cross-domain}, \ref{tab:teacher_number}, \ref{tab:prompt_technique}, \ref{tab:mapping_technique}, \ref{tab:SVHN}$\sim$~\ref{tab:DermaMNIST}, the row \textit{Rescale Ratio} means that the image from the target task is rescaled to a specific size. The row \textit{Queries} denotes the number of unlabeled samples that asks for the labels from the noisy aggregation result. Because Prom-PATE follows the design from \cite{scalablepate}, Confident-
GNMax does not always return a label. Thus, the row \textit{Answered Queries} shows the actual number of labels returned by the noisy aggregation result. The row \textit{Answer Accuracy} corresponds to the accuracy of the noisy label. The rows Threshold $T$, $\sigma_1$, and $\sigma_2$ are the parameters in Confident-
GNMax~\cite{scalablepate} that determines when a noisy label will be returned. 

In our experiments, we considered image datasets SVHN~\cite{karras2019style}, EuroSAT~\cite{helber2019eurosat}, PathMNIST~\cite{yang2023medmnist}, TissueMNIST~\cite{yang2023medmnist}, and DermaMNIST~\cite{yang2023medmnist} as our cross-domain tasks. SVHN contains the street-view house numbers. EuroSAT contains the Sentinel-2 satellite images for land use and land cover classification. PathMNIST, TissueMNIST, and DermaMNIST are collections of standardized biomedical images. In Table~\ref{tab:SVHN}$\sim$Table~\ref{tab:DermaMNIST}, we compare Prom-PATE with Transfer-PATE (see Section~\ref{subsec:cross-domain} for the setting of Transfer-PATE) to show the superiority of Prom-PATE. In particular, depending on the different characteristics of datasets, Prom-PATE reaches different accuracies. However, one can see that Transfer-PATE goes worse than Pro-PATE because the cross-domain task requires more fine-tuning and thus more samples in Transfer-PATE.  

\begin{table}[H]
    \begin{center}
    \resizebox{\columnwidth}{!}{
    \begin{tabular}{|c || c c|} 
     \hline
     \textbf{SVHN \cite{karras2019style} } & \textbf{Prom-PATE} & \textbf{Transfer-PATE} \\ [0.5ex] 
     \hline
     $\epsilon$ & 3.022 & 3.055\\ 
     \hline
     Rescale Ratio & 192$\times$192 & - \\ 
     \hline 
     Number of Re-Teachers & 250 & 250 \\ 
     \hline
     Source Model & WideResNet & WideResNet \\ 
     \hline 
     Queries & 2000 & 2000 \\
     \hline
     Answered Queries & 105 & 88 \\ 
     \hline
     Answer Accuracy($\%$) & 86.67 & 79.55 \\
     \hline
     Threshold T & 244 & 230 \\  
     \hline
     $\sigma_{1}$ & 50 & 50 \\ 
     \hline
     $\sigma_{2}$ & 10 & 10 \\
     \hline 
     \textbf{Accuracy ± Std($\%$)} & \textbf{49.13 $\pm$ 3.13} & 42.57 $\pm$ 0.23 \\
     \hline
    \end{tabular}}
    \end{center}
    \caption{SVHN Performance}
    \label{tab:SVHN}
\end{table}

\begin{table}[H]
    \begin{center}
    \resizebox{\columnwidth}{!}{
    \begin{tabular}{|c || c c|} 
     \hline
     \textbf{EuroSAT \cite{helber2019eurosat} } & \textbf{Prom-PATE} & \textbf{Transfer-PATE} \\ [0.5ex] 
     \hline
     $\epsilon$ & 3.018 & 3.029 \\ 
     \hline
     Rescale Ratio & 160$\times$160 & - \\ 
     \hline 
     Number of Re-Teachers & 250 & 250 \\ 
     \hline
     Source Model & WideResNet & WideResNet \\ 
     \hline 
     Queries & 1,000 & 1,000 \\
     \hline
     Answered Queries & 150 & 140 \\ 
     \hline
     Answer Accuracy(\%) & 90.7 & 90.7 \\
     \hline
     Threshold T & 250 & 250 \\  
     \hline
     $\sigma_{1}$ & 100 & 100 \\ 
     \hline
     $\sigma_{2}$ & 10 & 10 \\
     \hline 
     \textbf{Accuracy ± Std($\%$)} & \textbf{94.37 $\pm$ 0.15} & 93.87 $\pm$ 0.80 \\
     \hline
    \end{tabular}}
    \end{center}
    \caption{EuroSAT Performance.}
    \label{tab:EuroSAT}
\end{table}

\begin{table}[H]
    \begin{center}
    \resizebox{\columnwidth}{!}{
    \begin{tabular}{|c || c c|} 
     \hline
     \textbf{PathMNIST \cite{yang2023medmnist} } & \textbf{Prom-PATE} & \textbf{Transfer-PATE} \\ [0.5ex] 
     \hline
     $\epsilon$ & 1.024 & 1.021 \\ 
     \hline
     Rescale Ratio & 128$\times$128 & - \\ 
     \hline 
     Number of Re-Teachers & 1,000 & 1,000 \\ 
     \hline
     Source Model & WideResNet & WideResNet \\ 
     \hline 
     Queries & 1,000 & 1,000 \\
     \hline
     Answered Queries & 91 & 83 \\ 
     \hline
     Answer Accuracy(\%) & 100 & 98.80 \\
     \hline
     Threshold T & 1030 & 1050 \\  
     \hline
     $\sigma_{1}$ & 100 & 100 \\ 
     \hline
     $\sigma_{2}$ & 50 & 50 \\
     \hline 
     \textbf{Accuracy ± Std($\%$)} & \textbf{68.50 $\pm$ 0.40} & 64.00 $\pm$ 2.07 \\
     \hline
    \end{tabular}
    }
    \end{center}
    \caption{PathMNIST Performance.}
    \label{tab:PathMNIST}
\end{table}

\begin{table}[H]
    \begin{center}
    \resizebox{\columnwidth}{!}{
    \begin{tabular}{|c || c c|} 
     \hline
     \textbf{TissueMNIST \cite{yang2023medmnist} } & \textbf{Prom-PATE} & \textbf{Transfer-PATE} \\ [0.5ex] 
     \hline
     $\epsilon$ & 2.008 & 2.017 \\ 
     \hline
     Rescale Ratio & 160$\times$160 & - \\ 
     \hline 
     Number of Re-Teachers & 1,000 & 1,000 \\ 
     \hline
     Source Model & WideResNet & WideResNet \\ 
     \hline 
      Queries & 2,000 & 2,000 \\
     \hline
     Answered Queries & 842 & 665 \\ 
     \hline
     Answer Accuracy(\%) & 71.85 & 72.86 \\
     \hline
     Threshold T & 650 & 630 \\  
     \hline
     $\sigma_{1}$ & 100 & 100 \\ 
     \hline
     $\sigma_{2}$ & 50 & 50 \\
     \hline 
     \textbf{Accuracy ± Std(\%)} & \textbf{49.87 $\pm$ 0.57} & 49.30 $\pm$ 0.56 \\
     \hline
    \end{tabular}}
    \end{center}
    \caption{TissueMNIST Performance.}
    \label{tab:TissueMNIST}
\end{table}

\begin{table}[H]
    \begin{center}
    \resizebox{\columnwidth}{!}{
    \begin{tabular}{|c || c c|} 
     \hline
     \textbf{DermaMNIST \cite{yang2023medmnist} } & \textbf{Prom-PATE} & \textbf{Transfer-PATE} \\ [0.5ex] 
     \hline
     $\epsilon$ & 1.861 & 1.852 \\ 
     \hline
     Rescale Ratio & 192$\times$192 & - \\ 
     \hline 
     Number of Re-Teachers & 500 & 500 \\ 
     \hline
     Source Model & WideResNet & WideResNet \\ 
     \hline 
     Queries & 1,000 & 1,000 \\
     \hline
     Answered Queries & 806 & 749 \\ 
     \hline
     Answer Accuracy(\%) & 61.17 & 61.28 \\
     \hline
     Threshold T & 300 & 300 \\  
     \hline
     $\sigma_{1}$ & 100 & 100 \\ 
     \hline
     $\sigma_{2}$ & 200 & 200 \\
     \hline 
     \textbf{Accuracy ± Std($\%$)} & \textbf{60.34 $\pm$ 0.31} & 59.60 $\pm$ 0.20 \\
     \hline
    \end{tabular}}
    \end{center}
    \caption{DermaMNIST Performance.}
    \label{tab:DermaMNIST}
\end{table}

\begin{table*}[hbt]
    \begin{center}
    \resizebox{2\columnwidth}{!}{
    \small
    \begin{tabular}{|c || c c c c c c|} 
     \hline
     \textbf{CelebA-Gender \cite{liu2018large} } & \textbf{Prom-PATE}& \textbf{Prom-PATE} & \textbf{Prom-PATE} & \textbf{Transfer-PATE} & \textbf{Transfer-PATE}  & \textbf{Transfer-PATE}\\ 
     \hline
     $\epsilon$ & 1.555 & 1.55 & 1.528  & 1.552 & 1.547 & 1.536\\ 
     \hline
     Rescale Ratio & 192$\times$192 & 192$\times$192 & 192$\times$192 & - & - & -\\ 
     \hline 
     Number of Re-Teachers & 1,000 & 1,000 & 2,000 & 1,000 & 1,000 & 2,000 \\ 
     \hline
     Source Model & Swin & ViT & Swin & Swin & ViT & Swin \\ 
     \hline 
     Queries & 1,000 & 1,000 & 1,000 & 1,000 & 1,000 & 1,000\\
     \hline
     Answered Queries & 707 & 669  & 794 & 696 & 673 & 789 \\ 
     \hline
     Answer Accuracy(\%) & 97.17 & 97.16 & 97.86 & 97.70 & 97.33 & 97.85 \\
     \hline
     Threshold T & 900 & 900 & 1,800 & 800 & 900 & 1,800\\  
     \hline
     $\sigma_{1}$ & 100 & 100 & 1,000 & 100 & 100 & 1,000\\ 
     \hline
     $\sigma_{2}$ & 200 & 200 & 500 & 200 & 200 & 500\\
     \hline 
     \textbf{FixMatch Accuracy ± Std($\%$)} & 93.17 $\pm$ 0.06 & 92.23 $\pm$ 0.25 & 92.83 $\pm$ 0.20 & 92.87 $\pm$ 0.15 & 91.83 $\pm$ 0.23 & 92.63 $\pm$ 0.35 \\
     \hline
     \textbf{FreeMatch Accuracy ± Std($\%$)} & 93.20 $\pm$ 0.10 & 91.77 $\pm$ 0.40 & 92.33 $\pm$ 0.42 & 92.83 $\pm$ 0.25 & 91.17 $\pm$ 0.21 & 92.33 $\pm$ 0.23\\
     \hline
    \end{tabular}
    }
    \end{center}
    \caption{CelebA-Gender Performance.}
    \label{tab:CelebA-Gender}
\end{table*}

\begin{table*}[hbt]
    \begin{center}
    \resizebox{2\columnwidth}{!}{
    \small
    \begin{tabular}{|c || c c c c c c|} 
     \hline
     \textbf{CelebA-Hair \cite{liu2018large} } & \textbf{Prom-PATE}& \textbf{Prom-PATE} & \textbf{Prom-PATE} & \textbf{Transfer-PATE} & \textbf{Transfer-PATE}  & \textbf{Transfer-PATE}\\ 
     \hline
     $\epsilon$ & 1.531 & 1.521 & 1.527  & 1.522 & 1.530 & 1.534\\ 
     \hline
     Rescale Ratio & 192$\times$192 & 192$\times$192 & 192$\times$192 & - & - & -\\ 
     \hline 
     Number of Re-Teachers & 1,000 & 1,000 & 2,000 & 1,000 & 1,000 & 2,000 \\ 
     \hline
     Source Model & Swin & ViT & Swin & Swin & ViT & Swin \\ 
     \hline 
     Queries & 1,000 & 1,000 & 2,000 & 1,000 & 1,000 & 2,000\\
     \hline
     Answered Queries & 378 & 359  & 517 & 376 & 330 & 535 \\ 
     \hline
     Answer Accuracy(\%) & 91.53 & 93.31 & 90.52 & 92.55 & 90.30 & 90.09 \\
     \hline
     Threshold T & 780 & 800 & 1,500 & 790 & 780 & 1,500\\  
     \hline
     $\sigma_{1}$ & 100 & 100 & 100 & 100 & 100 & 100\\ 
     \hline
     $\sigma_{2}$ & 200 & 200 & 450 & 200 & 200 & 500\\
     \hline 
     \textbf{FixMatch Accuracy ± Std(\%)} & 85.40 $\pm$ 0.40 & 87.73 $\pm$ 0.15 & 81.97 $\pm$ 0.67 & 84.73 $\pm$ 0.83 & 84.77 $\pm$ 0.31 & 81.20 $\pm$ 0.17\\
     \hline
     \textbf{FreeMatch Accuracy ± Std(\%)} & 86.23 $\pm$ 0.45 & 88.13 $\pm$ 0.35 & 83.97 $\pm$ 0.47 & 87.23 $\pm$ 0.30 & 85.03 $\pm$ 0.46 & 83.80 $\pm$ 0.53\\
     \hline 
    \end{tabular}
    }
    \end{center}
    \caption{CelebA-Hair Performance.}
    \label{tab:CelebA-Hair}
\end{table*}

\begin{table*}[hbt]
    \begin{center}
    \resizebox{2\columnwidth}{!}{
    \small
    \begin{tabular}{|c || c c c c c c|} 
     \hline
     \textbf{FFHQ-Gender \cite{karras2019style} } & \textbf{Prom-PATE}& \textbf{Prom-PATE} & \textbf{Prom-PATE} & \textbf{Transfer-PATE} & \textbf{Transfer-PATE}  & \textbf{Transfer-PATE}\\ 
     \hline
     $\epsilon$ & 1.599 & 1.605 & 1.52  & 1.604 & 1.602 & 1.562 \\ 
     \hline
     Rescale Ratio & 192$\times$192 & 192$\times$192 & 192$\times$192 & - & - & -\\ 
     \hline 
     Number of Re-Teachers & 1,000 & 1,000 & 2,000 & 1,000 & 1,000 & 2,000 \\ 
     \hline
     Source Model & Swin & ViT & Swin & Swin & ViT & Swin \\ 
     \hline 
     Queries & 1,000 & 1,000 & 2,000 & 1,000 & 1,000 & 2,000\\
     \hline
     Answered Queries & 663 & 630  & 1,331 & 653 & 620 & 1,270 \\ 
     \hline
     Answer Accuracy(\%) & 95.47 & 93.33 & 94.21 & 94.79 & 94.52 & 94.17 \\
     \hline
     Threshold T & 800 & 800 & 1,500 & 800 & 800 & 1,500\\  
     \hline
     $\sigma_{1}$ & 100 & 100 & 90 & 100 & 100 & 100\\ 
     \hline
     $\sigma_{2}$ & 200 & 200 & 450 & 200 & 200 & 500\\
     \hline 
     \textbf{FixMatch Accuracy ± Std($\%$)} & 86.93 $\pm$ 0.21 & 86.13 $\pm$ 0.25 & 85.87 $\pm$ 0.35 & 86.07 $\pm$ 0.06 & 84.47 $\pm$ 0.21 & 84.83 $\pm$ 0.40\\
     \hline
     \textbf{FreeMatch Accuracy ± Std($\%$)} & 86.77 $\pm$ 0.21 & 86.47 $\pm$ 0.05 & 86.43 $\pm$ 0.31 & 86.47 $\pm$ 0.15 & 85.50 $\pm$ 0.36 & 84.90 $\pm$ 0.98\\
     \hline 
    \end{tabular}
    }
    \end{center}
    \caption{FFHQ-Gender Performance.}
    \label{tab:FFHQ-Gender}
\end{table*}

\paragraph{Additional Experiment Results for High-Resolution Images.} Here, we present additional experiment results for high-resolution images. CelebA is a popular dataset that contains colorful celebrity images of different sizes. All of CelebA images were rescaled into $64\times 64$ colorful images in our experiments. Based on CelebA, we consider CelebA-Gender and CelebA-Hair. In particular, CelebA-Gender is for binary classification with gender as the label. CelebA-Hair is for the three-class classification with hair color (black/blonde/brown) as the label. On the other hand, FFHQ contains 70000 $128\times 128$ colorful facial images with gender as labels. 

With the comparison between Table~\ref{tab:compare} and Tables~\ref{tab:CelebA-Gender}$\sim$~\ref{tab:FFHQ-Gender}, we can see that when images of higher resolutions are considered, Prom-PATE reaches a lower accuracy. A potential explanation is that because the images from the target task will be rescaled, certain features in the images from the target task will disappear and consequently have a negative impact on the resulting accuracy. 

With the comparison among Tables~\ref{tab:CelebA-Gender}, \ref{tab:CelebA-Hair}, and ~\ref{tab:FFHQ-Gender}, we can see that Prom-PATE outperforms Transfer-PATE in nearly all cases. Nonetheless, we can also find in Table~\ref{tab:CelebA-Gender} that Prom-PATE only slightly outperforms Transfer-PATE. Such a minor victory of Prom-PATE comes from the fact that the number of labels returned by the noisy aggregation in Prom-PATE does not have a significant increase compared to that in Transfer-PATE (e.g., see the row \textit{Answered Queries} in Table~\ref{tab:CelebA-Gender}). The similarity between the numbers of labels returned by the noisy aggregation in Prom-PATE and Transfer-PATE can be attributed to the similar training result of the teacher models. A potential explanation for such a phenomenon is that the CelebA-Gender is a binary classification task, which is easy for both Prom-PATE and Transfer-PATE. By comparing Table~\ref{tab:CelebA-Gender} and Table~\ref{tab:CelebA-Hair}, we can see that the accuracy difference between Prom-PATE and Transfer-PATE is becoming clear because CelebA-Hair is a three-class classification task. When considering Table~\ref{tab:cross-domain} and Table~\ref{tab:CelebA-Hair}, despite the different context in CIFAR-10 and CelebA-Hair, we can use them as evidence for claiming that Prom-PATE works especially better than Transfer-PATE in the multi-class classification task. 

Tables~\ref{tab:CelebA-Gender}, \ref{tab:CelebA-Hair}, and ~\ref{tab:FFHQ-Gender} also report the accuracies when two different semi-supervised learning methods, FixMatch \cite{sohn2020fixmatch} and FreeMatch  \cite{wang2022freematch}, are used. We can only see the minor difference between using two SOTA semi-supervised learning methods.

\paragraph{Additional Experiment Result for Benchmark Task}

\begin{table*}[hbt]
    \label{appx:cifar100}
    \begin{center}
    \addtolength{\tabcolsep}{8pt}
    \resizebox{\columnwidth}{!}{
    \begin{tabular}{|c |c |c |c |} 
     \hline
      \textbf{CIFAR-100 \cite{krizhevsky2009learning}}& $\epsilon$ & Sanitized $\epsilon$ & Accuracy\\ 
      \hline 
      \multirow{2}{*}{De et al.~\cite{de2022unlocking}} & 4 & 4 & 79.2\% \\ 
      & 8 & 8 & 81.8\% \\
      \hline 
      \multirow{2}{*}{Bu et al.~\cite{bu2022scalable}} & 4 & 4 & 87.7\% \\ 
      & 8 & 8 & 88.4\% \\ 
      \hline
      \multirow{2}{*}{Prom-PATE} & 4.089 & 5.043 & \textbf{88.33\%} \\ 
      & 8.078 & 10.026 & \textbf{91.47\%} \\ 
      \hline 
    \end{tabular}}
    \end{center}
    \vspace{-2mm}
    \caption{Comparison among DP classifiers on CIFAR-100.}
    \label{tab:cifar-100}
    \vspace{-5mm}
\end{table*}

Here, we present an additional experiment conducted on CIFAR-100 which is considered as the new benchmark task for DP classifiers. The result and comparison against SOTA classifiers are shown in Table~\ref{tab:cifar-100}. Here, the pre-trained model for re-teachers, the pre-trained model for semi-supervised learning, and the algorithm for semi-supervised learning of Prom-PATE in Table~\ref{tab:cifar-100} are EVA~\cite{fang2023eva}, EVA~\cite{fang2023eva}, and FreeMatch~\cite{wang2022freematch}, respectively.

As one can see from Table~\ref{tab:cifar-100}, Prom-PATE achieves SOTA under a similar budget comparison against other methods, demonstrating the benefits of exploring model reprogramming for parameter-efficient fine-tuning \cite{ding2023parameter} of DP models.

\end{document}